\documentclass{article} 
\usepackage{nips13submit_e,times}
\usepackage{url}
\usepackage{natbib}
\usepackage[utf8]{inputenc}
\usepackage{amsfonts}
\usepackage{amsmath}
\usepackage{amsthm}
\usepackage{algorithm}
\usepackage[noend]{algorithmic}
\usepackage{graphicx}
\newtheorem{proposition}{Proposition}

\newtheorem{corollary}{Corollary}
\newtheorem{theorem}{Theorem}

\nipsfinaltrue

\newcommand{\vs}[1]{\vspace*{-#1mm}}
\newcommand{\one}{\mathbf{1}}
\newcommand{\sigm}{{\rm sigm}}
\newcommand{\splus}{{\rm s}_+}
\newcommand*{\defeq}{\stackrel{\text{def}}{=}}

\title{Estimating or Propagating Gradients Through Stochastic Neurons
for Conditional Computation}

\newif\ifLong
\Longfalse
\newif\ifConvergence
\Convergencefalse

\author{
Yoshua Bengio, Nicholas Léonard and Aaron Courville\\
Département d'informatique et recherche opérationnelle\\
Université de Montréal\\
}

\begin{document}

\maketitle
\vs{3}
\begin{abstract}
  Stochastic neurons and hard non-linearities can be useful for a number of
  reasons in deep learning models, but in many cases they pose a
  challenging problem: how to estimate the gradient of a loss function with
  respect to the input of such stochastic or non-smooth neurons? I.e., can
  we ``back-propagate'' through these stochastic neurons?  We examine this
  question, existing approaches, and compare four families of
  solutions, applicable in different settings.  One of them is the
  minimum variance unbiased gradient estimator for stochatic binary neurons
  (a special case of the REINFORCE algorithm).
  A second approach, introduced here,
  decomposes the operation of a binary stochastic neuron into a
  stochastic binary part and a smooth differentiable part, which
  approximates the expected effect of the pure stochatic binary neuron to
  first order. A third approach involves the 
  injection of additive or multiplicative noise in a computational graph that 
is otherwise differentiable. A fourth approach heuristically copies the
gradient with respect to the stochastic output directly as an estimator of the gradient
with respect to the sigmoid argument (we call this the straight-through estimator).
To explore a context where these estimators are useful, 
we consider a small-scale version of 
{\em conditional computation}, where sparse stochastic units form a
distributed representation of gaters that can turn off in combinatorially
many ways large chunks of the computation performed in the rest of the
neural network. In this case, it is important that the gating units produce
an actual 0 most of the time. The resulting sparsity can be potentially be exploited to greatly reduce the computational cost of large deep networks for which conditional computation would be useful.
\end{abstract}

\vs{3}
\section{Introduction and Background}
\vs{2}

Many learning algorithms and in particular those based on neural networks 
or deep learning rely on gradient-based learning. To compute exact gradients,
it is better if the relationship between parameters and the training objective
is continuous and generally smooth. If it is only constant by parts, i.e., mostly flat, then
gradient-based learning is impractical. This was what motivated the
move from neural networks based on so-called formal neurons, with a hard
threshold output, to neural networks whose units are based on a
sigmoidal non-linearity, and the well-known back-propagation 
algorithm to compute the gradients~\citep{Rumelhart86b}.

We call the {\em computational graph} or {\em flow graph} the graph that
relates inputs and parameters to outputs and training criterion.  Although
it had been taken for granted by most researchers that smoothness of this
graph was a necessary condition for exact gradient-based training methods
to work well, recent successes of deep networks with rectifiers and other
``non-smooth''
non-linearities~\citep{Glorot+al-AI-2011-small,Krizhevsky-2012-small,Goodfellow+al-ICML2013-small}
clearly question that belief: see Section~\ref{sec:semi-hard} for a deeper
discussion.

In principle, even if there are hard decisions (such as the
treshold function typically found in formal neurons) in the computational
graph, it is possible to obtain {\em estimated gradients} by introducing
{\em perturbations} in the system and observing the effects. Although
finite-difference approximations of the gradient 
appear hopelessly inefficient (because independently perturbing
each of $N$ parameters to estimate its gradient would be $N$ times
more expensive than ordinary back-propagation), another option
is to introduce {\em random perturbations}, and this idea has
been pushed far (and experimented on neural networks for control)
by \citet{Spall-1992} with the
Simultaneous Perturbation Stochastic Approximation (SPSA) algorithm.

As discussed here (Section~\ref{sec:semi-hard}), non-smooth non-linearities
and stochastic perturbations can be combined to obtain reasonably
low-variance estimators of the gradient, and a good example of that success
is with the recent advances with {\em
  dropout}~\citep{Hinton-et-al-arxiv2012,Krizhevsky-2012,Goodfellow+al-ICML2013-small}. The
idea is to multiply the output of a non-linear unit by independent binomial
noise.  This noise injection is useful as a regularizer and it does slow
down training a bit, but not apparently by a lot (maybe 2-fold), which is very
encouraging. The symmetry-breaking and induced sparsity may also compensate
for the extra variance and possibly help to reduce ill-conditioning, as
hypothesized by~\citet{Bengio-arxiv-2013}.

However, it is appealing to consider noise whose amplitude can be modulated
by the signals computed in the computational graph, such as with
{\em stochastic binary neurons}, which output a 1 or a 0 according
to a sigmoid probability. Short of computing an average over an
exponential number of configurations, it would seem that computing the
exact gradient (with respect to the average of the loss over all possible
binary samplings of all the stochastic neurons in the neural network)
is impossible in such neural networks. The question is whether good
estimators (which might have bias and variance) or similar alternatives
can be computed and yield effective training. We discuss and compare here
four reasonable solutions to this problem, present theoretical
results about them, and small-scale experiments to validate that training
can be effective, in the context where one wants to use such stochastic
units to gate computation. The motivation, described further below, is
to exploit such sparse stochastic gating units for {\em conditional
computation}, i.e., avoiding to visit every parameter for every example,
thus allowing to train potentially much larger models for the same
computational cost.

\vs{2}
\subsection{More Motivations and Conditional Computation}
\vs{2}

One motivation for studying stochastic neurons is that stochastic behavior
may be a required ingredient in {\em modeling biological neurons}. The
apparent noise in neuronal spike trains could come from an actual noise
source or simply from the hard to reproduce changes in the set of input
spikes entering a neuron's dendrites. Until this question is resolved by
biological observations, it is interesting to study how such noise -- which
has motivated the Boltzmann machine~\citep{Hinton84} -- may impact
computation and learning in neural networks.

Stochastic neurons with binary outputs are also interesting because
they can easily give rise to {\em sparse representations} (that have many zeros), 
a form of regularization that has been used in many representation
learning algorithms~\citep{Bengio-Courville-Vincent-TPAMI2013}.
Sparsity of the representation corresponds to the prior that, for a given
input scene, most of the explanatory factors are irrelevant (and that
would be represented by many zeros in the representation). 

As argued by~\citet{Bengio-arxiv-2013}, sparse representations may
be a useful ingredient of {\em conditional computation}, by which
only a small subset of the model parameters are ``activated'' (and need
to be visited) for any particular example, thereby greatly reducing the
number of computations needed per example. Sparse gating units may
be trained to select which part of the model actually need to be computed
for a given example.

Binary representations are also useful as keys for a hash table, as in the
semantic hashing algorithm~\citep{Salakhutdinov+Geoff-2009}.  
Trainable stochastic neurons would also be useful inside recurrent networks
to take hard stochastic decisions about temporal events at different
time scales. This would be useful to train multi-scale temporal hierarchies
\footnote{Daan Wierstra, personal communication} such that back-propagated
gradients could quickly flow through the slower time scales. Multi-scale
temporal hierarchies for recurrent nets have already been proposed and
involved exponential moving averages of different time
constants~\citep{ElHihi+Bengio-nips8}, where each unit is still updated
after each time step.  Instead, identifying key events at a high level of
abstraction would allow these high-level units to only be updated when
needed (asynchronously), creating a fast track (short path) for gradient
propagation through time.

\vs{2}
\subsection{Prior Work}
\vs{2}

The idea of having stochastic neuron models is of course very old, with one
of the major family of algorithms relying on such neurons being the
Boltzmann machine~\citep{Hinton84}.  Another biologically motivated
proposal for synaptic strength learning was proposed by
\citet{Fiete+Seung-2006}. It is based on small zero-mean {\em i.i.d.}
perturbations applied at each stochastic neuron potential (prior to a
non-linearity) and a Taylor expansion of the expected reward as a function
of these variations. \citet{Fiete+Seung-2006} end up proposing a gradient
estimator that looks like a {\em correlation between the reward and the
  perturbation}, very similar to that presented in
Section~\ref{sec:unbiased}, and which is a special case of the
REINFORCE~\citep{Williams-1992} algorithm. However, unlike REINFORCE, their
estimator is only unbiased in the limit of small perturbations.

Gradient estimators based on stochastic perturbations have long been shown
to be much more efficient than standard finite-difference
approximations~\citep{Spall-1992}. Consider $N$ quantities $u_i$ to 
be adjusted in order to minimize an expected loss $L(u)$. A finite difference
approximation is based on measuring separately the effect of changing each
one of the parameters, e.g., through $\frac{L(u) - L(u - \epsilon e_i)}{\epsilon}$,
or even better, through $\frac{L(u + \epsilon e_i) - L(u - \epsilon e_i)}{2 \epsilon}$,
where $e_i = (0,0,\cdots,1,0,0,\cdots,0)$ where the 1 is at position $i$. With $N$
quantities (and typically $O(N)$ computations to calculate $L(u)$), the
computational cost of the gradient estimator is $O(N^2)$. Instead, a perturbation-based
estimator such as found in 
Simultaneous Perturbation Stochastic Approximation (SPSA) \citep{Spall-1992}
chooses a random perturbation vector $z$ (e.g., isotropic Gaussian noise
of variance $\sigma^2$) and estimates the gradient of the expected loss with
respect to $u_i$ through $\frac{L(u + z) - L(u - z)}{2 z_i}$.
So long as the perturbation does not put too much probability around 0,
this estimator is as efficient as the finite-difference estimator but
requires $O(N)$ less computation. However, like the algorithm proposed
by \citet{Fiete+Seung-2006} this estimator becomes unbiased only
as the perturbations go towards 0. When we want to consider all-or-none
perturbations (like a neuron sending a spike or not), it is not clear
if these assumptions are appropriate. An advantage of the approaches
proposed here is that they do not require that the perturbations
be small to be valid.

\vs{3}
\section{Non-Smooth Stochastic Neurons}
\label{sec:semi-hard}
\vs{2}

One way to achieve gradient-based learning in networks of stochastic neurons 
is to build an architecture in which noise is injected so that gradients
can sometimes flow into the neuron and can then adjust it (and its predecessors
in the computational graph) appropriately.

In general, we can consider the output $h_i$ of a stochastic neuron as the
application of a deterministic function that depends on some noise source $z_i$
and on some differentiable
transformation $a_i$ of its inputs
(typically, a vector containing the outputs of other neurons) and internal parameters
(typically the bias and incoming weights of the neuron):
\begin{equation}
\label{eq:noisy-output}
  h_i = f(a_i,z_i) 
\end{equation}
So long as $f(a_i,z_i)$ in the above equation has a non-zero
gradient with respect to $a_i$, gradient-based learning 
(with back-propagation to compute gradients) can proceed.
In this paper, we consider the usual affine transformation 
$a_i = b_i + \sum_j W_{ij} x_{ij}$, where $x_i$ is the vector
of inputs into neuron $i$.

For example, if the noise $z_i$ is added or multiplied somewhere in the
computation of $h_i$, gradients can be computed as usual. Dropout
noise~\citep{Hinton-et-al-arxiv2012} and masking noise (in denoising
auto-encoders~\citep{VincentPLarochelleH2008-small}) is multiplied (just
after the neuron non-linearity), while in semantic
hashing~\citep{Salakhutdinov+Geoff-2009} noise is added (just before the
non-linearity). For example, that noise can be binomial (dropout, masking
noise) or Gaussian (semantic hashing).

However, if we want $h_i$ to be binary (like in stochastic binary neurons), 
then $f$ will have derivatives that are 0 almost everywhere (and infinite
at the threshold), so that gradients can never flow. 

There is an intermediate option that we put forward here: choose $f$ so
that it has two main kinds of behavior, with zero derivatives in some
regions, and with significantly non-zero derivatives in other regions.  
We call these two states of the neuron respectively the insensitive state
and the sensitive state. 

\vs{3}
\subsection{Noisy Rectifier}
\vs{2}

A special case is when the insensitive state corresponds to $h_i=0$ and we
have sparsely activated neurons. The prototypical example of that situation
is the rectifier unit~\citep{Hinton2010,Glorot+al-AI-2011-small}, whose
non-linearity is simply $\max(0,{\rm arg})$, i.e.,
\[
  h_i = \max(0, z_i + a_i)
\]
where $z_i$ is a zero-mean noise. Although in practice Gaussian noise is
appropriate, interesting properties can be shown for $z_i$ sampled from
the logistic density $p(z)=\sigm(z)(1-\sigm(z))$, where $\sigm(\cdot)$
represents the logistic sigmoid function.

In that case, we can prove nice properties of the noisy rectifier.
\begin{proposition}
\label{thm:noisy-rectifier}
The noisy rectifier with logistic noise has the following properties: (a) $P(h_i>0)=\sigm(a_i)$,
(b) $E[h_i]=\splus(a_i)$, where $\splus(x)=\log(1+\exp(x))$ is the softplus function
and the random variable is $z_i$ (given a fixed $a_i$).
\end{proposition}
\begin{proof}
\begin{eqnarray*}
 P(h_i>0) &=& P(z_i>-a_i) = 1-\sigm(-a_i)\\
  &=& 1 - \frac{1}{1 + e^{a_i}} = \frac{1}{1+e^{-a_i}} = \sigm(a_i)
\end{eqnarray*}
which proves (a). For (b), we use the facts $\frac{d \,\sigm(x)}{dx}=\sigm(x)(1-\sigm(x))$,
$\splus(u) \rightarrow u$ as $u\rightarrow\infty$,
change of variable $x=a_i+z_i$ and integrate by parts:
\begin{eqnarray*}
 E[h_i] &=& \lim_{t\rightarrow\infty} \int_{z_i>-a_i}^t (a_i+z_i) \sigm(z_i)(1-\sigm(z_i) dz_i \\
    &=& \lim_{t\rightarrow\infty} \left(\left[x \,\sigm(x-a_i)\right]_0^t - \int_0^t \sigm(x-a)dx \right)\\
    &=& \lim_{t\rightarrow\infty} \left[x \,\sigm(x-a_i)-\splus(x-a_i)\right]_0^t \\
    &=& a_i + \splus(-a_i) = \splus(a_i)
\end{eqnarray*}
\end{proof}

Let us now consider two cases:
\begin{enumerate}
\item If $f(a_i,0)>0$, the basic state is {\em active},
the unit is generally sensitive and non-zero,
but sometimes it is shut off (e.g., when $z_i$ is sufficiently negative to push 
the argument of the rectifier below 0). In that case gradients will flow
in most cases (samples of $z_i$). If the rest of the system sends the signal 
that $h_i$ should have been smaller, then gradients will push it towards
being more often in the insensitive state.
\item If $f(a_i,0)=0$, the basic state is {\em inactive},
  the unit is generally insensitive and zero,
  but sometimes turned on (e.g., when $z_i$ is sufficiently positive to push the
  argument of the rectifier above 0). In that case gradients will not flow
  in most cases, but when they do, the signal will either push the weighted
  sum lower (if being active was not actually a good thing for that unit in
  that context) and reduce the chances of being active again, or it will
  push the weight sum higher (if being active was actually a good thing for
  that unit in that context) and increase the chances of being active
  again.
\end{enumerate}

So it appears that even though the gradient does not always flow
(as it would with sigmoid or tanh units), it might flow sufficiently
often to provide the appropriate training information. The important
thing to notice is that even when the basic state (second case, above)
is for the unit to be insensitive and zero, {\em there will be
an occasional gradient signal} that can draw it out of there.

One concern with this approach is that one can see an asymmetry between
the number of times that a unit with an active state can get a chance
to receive a signal telling it to become inactive, versus the number of
times that a unit with an inactive state can get a signal telling
it to become active. 

Another potential and related concern is that some of these units will
``die'' (become useless) if their basic state is inactive in the vast
majority of cases (for example, because their weights pushed them into that
regime due to random fluctations). Because of the above asymmetry, dead
units would stay dead for very long before getting a chance of being born
again, while live units would sometimes get into the death zone by chance
and then get stuck there. What we propose here is a simple mechanism to
{\em adjust the bias of each unit} so that in average its ``firing rate''
(fraction of the time spent in the active state) reaches some pre-defined
target. For example, if the moving average of being non-zero falls below a
threshold, the bias is pushed up until that average comes back above the
threshold.

\vs{3}
\subsection{STS Units: Stochastic Times Smooth}
\vs{2}

We propose here a novel form of stochastic unit that is related to
the noisy rectifier and to the stochastic binary unit, but that can
be trained by ordinary gradient descent with the gradient obtained
by back-propagation in a computational graph. We call it the
STS unit (Stochastic Times Smooth).
With $a_i$ the
activation before the stochastic non-linearity, the output of
the STS unit is
\begin{align}
\label{eq:sts-unit}
 p_i =& \; \sigm(a_i) \nonumber \\
 b_i \sim& \; {\rm Binomial}(\sqrt{p_i}) \nonumber \\
 h_i =& \; b_i \sqrt{p_i}
\end{align}

\begin{proposition}
\label{thm:sts-unit}
The STS unit benefits from the following properties: (a) $E[h_i]=p_i$,
(b) \mbox{$P(h_i>0)=\sqrt{\sigm(a_i)}$}, (c) for any differentiable function $f$
of $h_i$, we have
\[
   E[f(h_i)] = f(p_i) + o(\sqrt{p_i})
\]
as $p_i\rightarrow 0$, and 
where the expectations and probability are over the injected noise $b_i$, given $a_i$.
\end{proposition}
\begin{proof}
Statements (a) and (b) are immediately derived from the definitions
in Eq.~\ref{eq:sts-unit} by noting that $E[b_i]=P(b_i>0)=P(h_i>0)=\sqrt{p_i}$.
Statement (c) is obtained by first writing out the expectation,
\[
  E[f(h_i)] = \sqrt{p_i}f(\sqrt{p_i})+(1-\sqrt{p_i})f(0)
\]
and then performing a Taylor expansion of $f$ around $p_i$ as $p_i \rightarrow 0$):
\begin{eqnarray*}
  f(\sqrt{p_i}) &=& f(p_i) + f'(p_i)(\sqrt{p_i}-p_i) + o(\sqrt{p_i}-p_i) \\
  f(0) &=& f(p_i) + f'(p_i)(-p_i) + o(p_i) 
\end{eqnarray*}
where $\frac{o(x)}{x} \rightarrow 0$ as $x \rightarrow 0$, so we obtain
\begin{eqnarray*}
 E[f(h_i)] &=& \sqrt{p_i}(f(p_i)+f'(p_i)(\sqrt{p_i}-p_i))+(1-\sqrt{p_i})(f(p_i) + f'(p_i)(-p_i)) + o(\sqrt{p_i}) \\
   &=& \sqrt{p_i} f(p_i)+f'(p_i)(p_i-p_i\sqrt{p_i})
     +f(p_i)(1-\sqrt{p_i}) - f'(p_i)(p_i - p_i \sqrt{p_i}) + o(\sqrt{p_i}) \\
   &=& f(p_i) + o(\sqrt{p_i}) 
\end{eqnarray*}
using the fact that $\sqrt{p_i}>p_i$, $o(p_i)$ can be replaced by $o(\sqrt{p_i})$.
\end{proof}
Note that this derivation can be generalized to an expansion around $f(x)$ for
any $x\leq \sqrt{p_i}$, yielding
\[
   E[f(h_i)] = f(x) + (x-p_i)f'(x) + o(\sqrt{p_i}) 
\]
where the effect of the first derivative is canceled out when we choose $x=p_i$.
It is also a reasonable choice to expand around $p_i$ because $E[h_i]=p_i$.

\vs{3}
\section{Unbiased Estimator of Gradient for Stochastic Binary Neurons}
\label{sec:unbiased}
\vs{2}

The above proposals cannot deal with non-linearities like
the indicator function that have a derivative of 0 almost everywhere.
Let us consider this case now, where we want some component of our model
to take a hard binary decision but allow this decision to be stochastic,
with a probability that is a continuous function of some 
quantities, through parameters that we wish to learn. 
We will also assume that many such decisions can be taken
in parallel with independent noise sources $z_i$ driving the stochastic
samples. Without loss of generality, we consider here a
set of binary decisions, i.e., the setup corresponds to
having a set of stochastic binary neurons, whose output $h_i$
influences an observed future loss $L$. In the framework of
Eq.~\ref{eq:noisy-output}, we could have for example
\begin{equation}
\label{eq:stochastic-binary-neuron}
  h_i = f(a_i, z_i) = \one_{z_i > \sigm(a_i)}
\end{equation}
where $z_i \sim U[0,1]$ is uniform and $\sigm(u)=1/(1+\exp(-u))$
is the sigmoid function. 
We would ideally like to estimate
how a change in $a_i$ would impact $L$ in average over the
noise sources, so as to be able to propagate this estimated
gradients into parameters and inputs of the stochastic neuron.
\begin{theorem}
\label{thm:unbiased-estimator}
Let $h_i$ be defined as in Eq.~\ref{eq:stochastic-binary-neuron},
with $L=L(h_i,c_i,c_{-i})$ a loss that depends stochastically on $h_i$, $c_i$ the
noise sources that influence $a_i$, and $c_{-i}$ those that do not
influence $a_i$,
then \mbox{$\hat{g}_i=(h_i - \sigm(a_i)) \times L$}
is an unbiased estimator of $g_i=\frac{\partial
  E_{z_i,c_{-i}}[L|c_i]}{\partial a_i}$ where the expectation is over $z_i$
and $c_{-i}$, conditioned on the set of
noise sources $c_i$ that influence $a_i$.
\vs{3}
\end{theorem}

\begin{proof}
We will compute the expected value of the estimator and verify
that it equals the desired derivative. The set of all noise sources
in the system is $\{z_i\} \cup c_i \cup c_{-i}$. We can consider
$L$ to be an implicit deterministic function of all the noise sources
$(z_i,c_i,c_{-i})$, $E_{v_z}[\cdot]$ denotes the expectation over variable $v_z$,
while $E[\cdot|v_{z}]$ denotes
the expectation over all the other random variables besides $v_z$, i.e., conditioned on $v_Z$.
\begin{align}
 E[L|c_i] &= E_{c_{-i}}[E_{z_i}[L(h_i,c_i,c_{-i})]] \nonumber \\
                    &= E_{c_{-i}}[E_{z_i}[h_i L(1,c_i,c_{-i})+(1-h_i) L(0,c_i,c_{-i})]] \nonumber \\
                    &=E_{c_{-i}}[P(h_i=1|a_i) L(1,c_i,c_{-i})+P(h_i=0|a_i) L(0,c_i,c_{-i})] \nonumber \\
                    &=E_{c_{-i}}[\sigm(a_i) L(1,c_i,c_{-i})+(1-\sigm(a_i)) L(0,c_i,c_{-i})] 
\end{align}
Since $a_i$ does not influence $P(c_{-i})$, differentiating with respect to $a_i$ gives
\begin{align}
 g_i \defeq \frac{\partial E[L|c_i]}{\partial a_i} &= 
   E_{c_{-i}}[\frac{\partial \sigm(a_i)}{\partial a_i} L(1,c_i,c_{-i})-
            \frac{\partial \sigm(a_i)}{\partial a_i} L(0,c_i,c_{-i}) | c_i]  \nonumber \\
 &=  E_{c_{-i}}[\sigm(a_i)(1-\sigm(a_i))(L(1,c_i,c_{-i})-L(0,c_i,c_{-i}) | c_i] 
\label{eq:gradient}
\end{align}
First consider that since  $h_i \in \{0,1\}$,
\[
 L(h_i,c_i,c_{-i}) = h_i L(1,c_i,c_{-i}) + (1-h_i) L(0,c_i,c_{-i})
\]
$h_i^2=h_i$ and $h_i(1-h_i)=0$, so
\begin{align}
\hat{g}_i \defeq (h_i - \sigm(a_i)) L(h_i,c_i,c_{-i}) &= h_i(h_i -
\sigm(a_i)) L(1,c_i,c_{-i}) \\
& \hspace{1cm} + (h_i-\sigm(a_i))(1-h_i) L(0,c_i,c_{-i})) \nonumber \\
  &= h_i(1 - \sigm(a_i)) L(1,c_i,c_{-i}) \\
& \hspace{1cm} - (1-h_i) \sigm(a_i) L(0,c_i,c_{-i}).
\end{align}
Now let us consider the expected value of the estimator $\hat{g}_i=(h_i - \sigm(a_i)) L(h_i,c_i,c_{-i})$.
\begin{align}{ll}
 E[\hat{g}_i] &= 
   E[h_i(1 - \sigm(a_i)) L(1,c_i,c_{-i}) - (1-h_i) \sigm(a_i) L(0,c_i,c_{-i})] \nonumber \\
  &=    E_{c_i,c_{-i}}[\sigm(a_i)(1 - \sigm(a_i)) L(1,c_i,c_{-i}) - (1-\sigm(a_i)) \sigm(a_i) L(0,c_i,c_{-i})] \nonumber \\
  &=    E_{c_i,c_{-i}}[\sigm(a_i)(1 - \sigm(a_i)) (L(1,c_i,c_{-i}) -L(0,c_i,c_{-i}))] 
\end{align}
which is the same as Eq.~\ref{eq:gradient}, i.e., the expected value of the
estimator equals the gradient of the expected loss, $E[\hat{g}_i]=g_i$.
\end{proof}
This estimator is a special case of the REINFORCE
algorithm when the stochastic unit is a Bernoulli with probability
given by a sigmoid~\citep{Williams-1992}. In the REINFORCE paper, Williams
shows that if the stochastic action $h$ is sampled with probability
$p_\theta(h)$ and yields a reward $R$, then 
\[
  E_h[(R-b) \cdot  \frac{\partial \log p_\theta(h)}{\partial \theta}] = 
  \frac{\partial E_h[R]}{\partial \theta}
\]
where $b$ is an arbitrary constant,
i.e., the sampled value $h$ can be seen as a weighted maximum
likelihood target for the output distribution $p_\theta(\cdot)$,
with the weights $R-b$ proportional to the reward. The additive
normalization constant $b$ does not change the expected gradient,
but influences its variance, and an optimal choice can be computed,
as shown below.

\begin{corollary}
\label{eq:corollary}
Under the same conditions as Theorem~\ref{thm:unbiased-estimator},
and for any (possibly unit-specific) constant $\bar{L}_i$ 
the centered estimator $(h_i - \sigm(a_i))(L - \bar{L}_i)$,
is also an unbiased estimator of 
$g_i=\frac{\partial E_{z_i,c_{-i}}[L|c_i]}{\partial a_i}$. 
Furthermore, among all possible values of $\bar{L}_i$, the minimum 
variance choice is 
\begin{equation}
\label{eq:opt-L}
 \bar{L}_i = \frac{E[(h_i-\sigm(a_i))^2 L]}{E[(h_i-\sigm(a_i))^2]},
\end{equation}
which we note is a weighted average of the loss values $L$, whose
weights are specific to unit $i$.
\end{corollary}
\begin{proof}
The centered estimator $(h_i - \sigm(a_i))(L - \bar{L}_i)$ can be decomposed
into the sum of the uncentered estimator $\hat{g}_i$ and the term
$(h_i - \sigm(a_i))\bar{L}_i$.  Since $E_{z_i}[h_i|a_i]=\sigm(a_i)$, 
$E[\bar{L}_i(h_i - \sigm(a_i))|a_i]=0$, so that
the expected value of the centered estimator equals the
expected value of the uncentered estimator. By
Theorem~\ref{thm:unbiased-estimator} (the uncentered estimator is
unbiased), the centered estimator is therefore also unbiased,
which completes the proof of the first statement.

Regarding the optimal choice of $\bar{L}_i$,
first note that the variance of the uncentered estimator is
\[
 Var[(h_i - \sigm(a_i))L]=E[(h_i - \sigm(a_i))^2L^2] - E[\hat{g}_i]^2.
\]
Now let us compute the variance of the centered estimator:
\begin{align}
  Var[(h_i - \sigm(a_i))(L - \bar{L}_i)] &= 
  E[(h_i - \sigm(a_i))^2(L - \bar{L}_i)^2] - E[(h_i-\sigma(a_i))(L-\bar{L}_i)]^2 \nonumber \\
 &= E[(h_i - \sigm(a_i))^2L^2] + E[(h_i-\sigm(a_i))^2\bar{L}_i^2] \nonumber \\
   & - 2 E[(h_i-\sigm(a_i))^2L\bar{L}_i] -(E[\hat{g}_i]-0)^2 \nonumber \\
 &= Var[(h_i - \sigm(a_i))L] - \Delta 
\end{align}
where $\Delta=2 E[(h_i-\sigm(a_i))^2 L\bar{L}_i] - E[(h_i-\sigm(a_i))^2\bar{L}_i^2]$.
Let us rewrite $\Delta$:
\begin{align}
 \Delta &= 2 E[(h_i-\sigm(a_i))^2 L \bar{L}_i] - E[(h_i-\sigm(a_i))^2\bar{L}_i^2] \nonumber \\
      &= E[(h_i-\sigm(a_i))^2 \bar{L}_i(2L - \bar{L}_i)] \nonumber \\
   &= E[(h_i-\sigm(a_i))^2 (L^2 - (L-\bar{L}_i)^2)] 
\end{align}
$\Delta$ is maximized (to minimize variance of the estimator)
when $E[(h_i-\sigm(a_i))^2 (L-\bar{L}_i)^2]$ is minimized. Taking the
derivative of that expression with respect to $\bar{L}_i$, we obtain
\[
\vs{2}
  2 E[(h_i-\sigm(a_i))^2 (\bar{L}_i-L)] =0
\]
which, as claimed, is achieved for 
\[
\vs{3}
  \bar{L}_i = \frac{E[(h_i-\sigm(a_i))^2 L]}{E[(h_i-\sigm(a_i))^2]}.
\vs{2}
\]
\vspace*{-4mm}
\end{proof}
Note that a general formula for the lowest variance estimator
for REINFORCE~\citep{Williams-1992} had already been introduced in the reinforcement
learning context by \citet{Weaver+Tao-UAI2001},
which includes the above result as a special case. This followed from
previous work~\citep{Dayan-1990} for the case of binary immediate reward.

Practically, we could get the lowest variance estimator (among all choices of the $\bar{L}_i$)
by keeping track of two numbers (running or moving averages) 
for each stochastic neuron, one for the numerator
and one for the denominator of the unit-specific $\bar{L}_i$ in Eq.~\ref{eq:opt-L}.
This would lead the lowest-variance estimator $(h_i - \sigm(a_i))(L - \bar{L}_i)$.
Note how the unbiased estimator only requires broadcasting $L$ throughout the
network, no back-propagation and only local computation.
Note also how this could be applied even with an estimate
of future rewards or losses $L$, as would be useful in the context of
reinforcement learning (where the actual loss or reward will be measured farther
into the future, much after $h_i$ has been sampled).

\section{Straight-Through Estimator}

Another estimator of the expected gradient through stochastic neurons
was proposed by \citet{Hinton-Coursera2012}
in his lecture 15b. The idea is simply to back-propagate
through the hard threshold function (1 if the argument is positive, 0 otherwise)
as if it had been the identity function. It is clearly a biased estimator,
but when considering a single layer of neurons, it has the right sign
(this is not guaranteed anymore when back-propagating through more hidden layers).
We call it the {\bf straight-through} (ST) estimator. A possible variant
investigated here multiplies the gradient on $h_i$ by the derivative of the
sigmoid. Better results were actually obtained without multiplying by
the derivative of the sigmoid. With $h_i$ sampled as per Eq.~\ref{eq:stochastic-binary-neuron},
the straight-through estimator  of the gradient of the loss $L$ with respect to the
pre-sigmoid activation $a_i$ is thus 
\begin{equation}
  g_i = \frac{\partial L}{\partial h_i}.
\end{equation}
Like the other estimators, it is then back-propagated to obtain gradients on the parameters
that influence $a_i$.

\begin{figure}[ht]
\centering
\vs{4}
\includegraphics[width=0.85\textwidth]{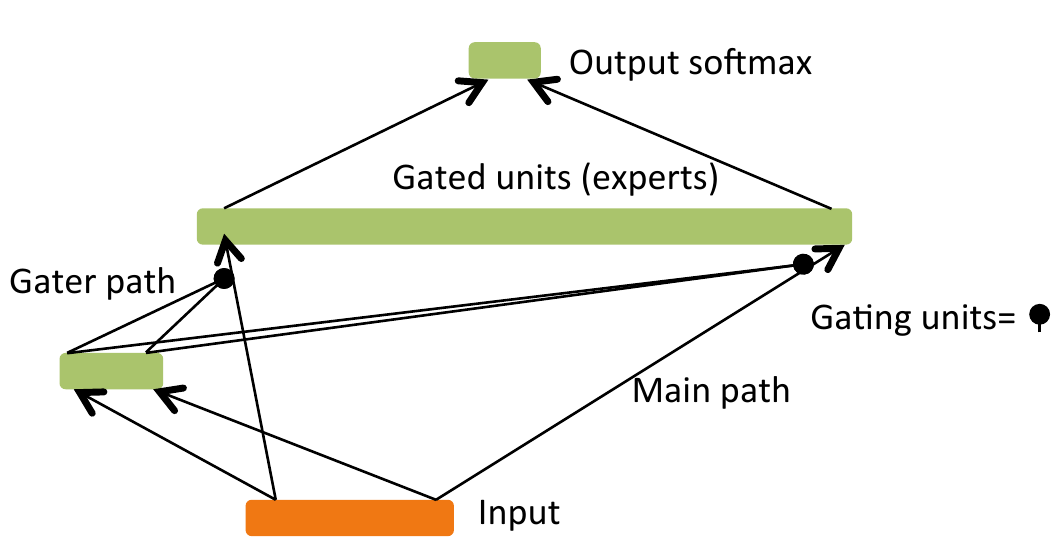}
\vs{3}
\caption{Conditional computation architecture: gater path on the left
produces sparse outputs $h_i$ which are multiplied elementwise by
the expert outputs, hidden units $H_i$.}
\label{fig:arch}
\vs{2}
\end{figure}

\vs{3}
\section{Conditional Computation Experiments}
\vs{2}

We consider the potential application of stochastic neurons in the
context of conditional computation, where the stochastic neurons are
used to select which parts of some computational graph should be
actually computed, given the current input. The particular architecture we
experimented with is a neural network with a large hidden layer whose
units $H_i$ will be selectively turned off by gating units $h_i$,
i.e., the output of the i-th hidden unit is $H_i h_i$, as illustrated
in Figure~\ref{fig:arch}.
For this to make sense in the context of conditional computation, we
want $h_i$ to be non-zero only a small fraction $\alpha$ of the time (10\% in the
experiments) while the amount of computation required to compute $h_i$
should be much less than the amount of computation required to compute $H_i$.
In this way, we can first compute $h_i$ and only compute $H_i$ if $h_i \neq 0$,
thereby saving much computation. To achieve this, we connect the previous
layer to $h_i$ through a bottleneck layer. Hence if the previous layer
(or the input) has size $N$ and the main path layer also has size $N$
(i.e., $i \in \{1,\ldots N\}$), and the gater's bottleneck layer has size $M \ll N$,
then computing the gater output is $O(MN)$, which can reduce the main path
computations from $O(N^2)$ to $O(\alpha N^2)$.

\begin{figure}[ht]
\centering
\vs{3}
\begin{minipage}{.4\textwidth}
\hspace*{-5mm}\includegraphics[width=1.3\textwidth]{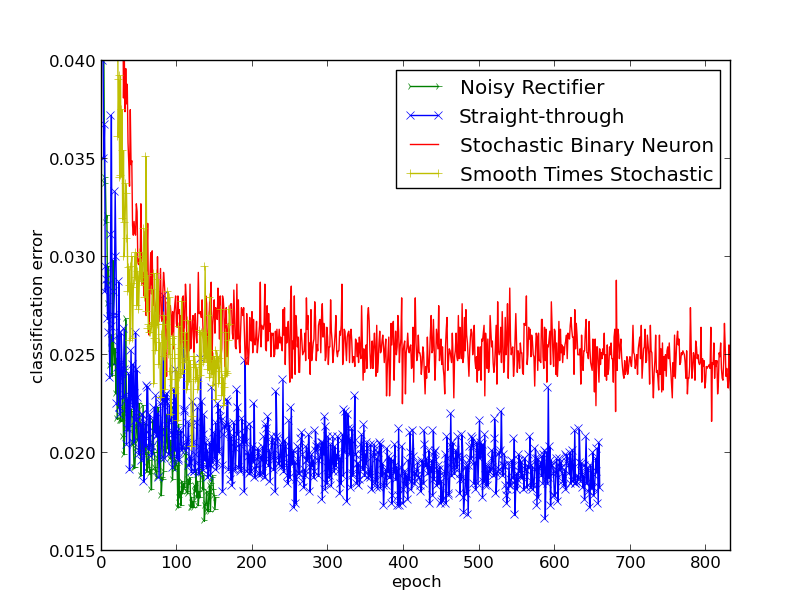}
\end{minipage} \hspace*{-1mm} %
\begin{minipage}{.59\textwidth}
  \centering
{\small
\hspace*{4mm}  \begin{tabular}{|l|ccc|}
    \hline
    & train & valid & test \\
    \hline
    Noisy Rectifier & 6.7e-4 & 1.52 & 1.87 \\
    Straight-through & 3.3e-3  & 1.42 & 1.39 \\
    Smooth Times Stoch. & 4.4e-3 & 1.86 & 1.96 \\
    Stoch. Binary Neuron & 9.9e-3 & 1.78 & 1.89 \\
    \hline
    Baseline Rectifier & 6.3e-5 & 1.66 & 1.60 \\
    Baseline Sigmoid+Noise & 1.8e-3 & 1.88 & 1.87\\
    Baseline Sigmoid &3.2e-3 & 1.97 & 1.92\\
    \hline
  \end{tabular}
}
\end{minipage} 
\vs{5}
\caption{\small Left: Learning curves for the various stochastic approaches applied to the MNIST dataset. The y-axis indicates classification error on the validation set. Right: Comparison of training criteria  and valid/test classification error (\%) for stochastic (top) and baseline (bottom) gaters. }
\label{fig:curves+results}
\vs{1}
\end{figure}

Experiments are performed with a gater of 400 hidden units and 2000 output
units. The main path also has 2000 hidden units. A sparsity constraint is imposed
on the 2000 gater output units such that each is non-zero 10\% of the time,
on average. The validation and test errors of the stochastic models
are obtained after optimizing a threshold in their deterministic counterpart used during testing.
See the appendix in supplementary material for more details on
the experiments.

For this experiment, we compare with 3 baselines. The Baseline Rectifier is
just like the noisy rectifier, but with 0 noise, and is also
constrained to have a sparsity of 10\%. The Baseline Sigmoid is like the
STS and ST models in that the gater uses sigmoids for its output and tanh
for its hidden units, but has no sparsity constraint and only 200 output
units for the gater and main part. The baseline has the same hypothetical
resource constraint in terms of computational efficiency (run-time), but
has less total parameters (memory). The Baseline Sigmoid with Noise is the
same, but with Gaussian noise added during training.


\vs{3}
\section{Conclusion}
\vs{2}

In this paper, we have motivated estimators of the gradient through highly
non-linear non-differentiable functions (such as those corresponding to an
indicator function), especially in networks involving noise sources, such
as neural networks with stochastic neurons. They can be useful as
biologically motivated models and they might be useful for engineering
(computational efficiency) reasons when trying to reduce computation via
conditional computation or to reduce interactions between parameters via
sparse updates~\citep{Bengio-arxiv-2013}. We have proven interesting properties
of three classes of stochastic neurons, the noisy rectifier, the STS unit,
and the binary stochastic neuron, in particular showing the existence
of an unbiased estimator of the gradient for the latter. 
Unlike the SPSA~\citep{Spall-1992} estimator, our estimator is unbiased
even though the perturbations are not small (0 or 1), and it multiplies
by the perturbation rather than dividing by it.

Experiments show that all the tested methods actually allow training to
proceed.  It is interesting to note that the gater with noisy rectifiers
yielded better results than the one with the non-noisy baseline
rectifiers. Similarly, the sigmoid baseline with noise performed better
than without, {\em even on the training objective}: all these results
suggest that injecting noise can be useful not just as a regularizer but
also to help explore good parameters and fit the training objective.  These
results are particularly surprising for the stochastic binary neuron, which
does not use any backprop for getting a signal into the gater, and opens
the door to applications where no backprop signal is available. In terms of
conditional computation, we see that the expected saving is achieved,
without an important loss in performance. Another surprise is the good
performance of the straight-through units, which provided the best
validation and test error, and are very simple to implement.

\ifnipsfinal
\subsubsection*{Acknowledgments}

The authors would like to acknowledge the useful comments from Guillaume
Alain and funding from NSERC, Ubisoft, CIFAR (YB is a CIFAR Fellow),
and the Canada Research Chairs.
\fi


\newpage

\bibliography{strings,strings-shorter,ml,aigaion-shorter}
\bibliographystyle{natbib}

\newpage

\ifnipsfinal
\appendix
\newif\ifsupplementary
\supplementaryfalse

\section{Details of the Experiments}

We frame our experiments using a conditional computation architecture. We
limit ourselves to a simple architecture with 4 affine transforms. The
output layer consists of an affine transform followed by softmax over the
10 MNIST classes. The input is sent to a gating subnetwork and to an experts
subnetwork, as in Figure
\ifsupplementary
1 of the main paper.
\else
\ref{fig:arch}.
\fi
There is one gating unit per expert unit, and the expert units are hidden
units on the main path. Each gating unit has a possibly stochastic
non-linearity (different under different algorithms evaluated here)
applied on top of an affine transformation of the gater path hidden layer
(400 tanh units that follow another affine transformation applied on the input).
The gating non-linearity are either Noisy Rectifiers, Smooth
Times Stochastic (STS), Stochastic Binary Neurons (SBN), Straight-through
(ST) or (non-noisy) Rectifiers over 2000 units.

The expert hidden units are obtained through an
affine transform without any non-linearity, which makes this part a simple linear
transform of the inputs into 2000 expert hidden units. Together, the gater and expert form a
conditional layer. These could be stacked to create deeper architectures,
but this was not attempted here. In our case, we use but one conditional
layer that takes its input from the vectorized 28x28 MNIST images. The
output of the conditional layer is the element-wise multiplication of the
(non-linear) output of the gater with the (linear) output of the
expert. The idea of using a linear transformation for the expert is derived
from an analogy over rectifiers which can be thought of as the product of a non-linear
gater ($\one_{h_i>0}$) and a linear expert ($h_i$).

\subsection{Sparsity Constraint}
Computational efficiency is gained by imposing a sparsity constraint on the
output of the gater. All experiments aim for an average sparsity of 10\%, such that
for 2000 expert hidden units we will only require computing approximately 200 of
them in average. Theoretically, efficiency can be
gained by only propagating the input activations to the selected expert
units, and only using these to compute the network output. 
For imposing the sparsity constraint we use a KL-divergence criterion for
sigmoids 
and an L1-norm criterion for rectifiers, where the amount of penalty
is adapted to achieve the target level of average sparsity.

A sparsity target of $s=0.1$ over units $i=1...2000$, each such unit having
a mean activation of $p_i$ within a mini-batch of 32 propagations, yields
the following KL-divergence training criterion:
\begin{equation}
{\rm KL}(s||p) = -\lambda \sum_i{\left(s\log{p_i} + (1-s)\log{1-p_i}\right)}
\end{equation}
where $\lambda$ is a hyper-parameter that can be optimized through
cross-validation. In the case of rectifiers, we use an L1-norm training
criteria:
\begin{equation}
{\rm L1}(p) = \lambda \sum{|p_i|}
\end{equation}
In order to keep the effective sparsity $s_e$ (the average proportion of
non-zero gater activations in a batch) of the rectifier near the target
sparsity of $s=0.1$, $\lambda$ is increased when $s_e > s+0.01$, and
reduced when $s_e < s-0.01$. This simple approach was found to be effective
at maintaining the desired sparsity.

\subsection{Beta Noise}
There is a tendency for the KL-divergence criterion to keep the sigmoids around the target
sparsity of $s=0.1$. This is not the kind of behavior desired from a gater, it indicates
indecisiveness on the part of the gating units. What we would hope to see is each unit maintain a 
mean activation of $0.1$ by producing sigmoidal values over 0.5 approximately 10\% of the time. 
This would allow us to round sigmoid values to zero or one at test time.

In order to encourage this behavior, we introduce noise at the input of the
sigmoid, as in the semantic hashing
algorithm~\citep{Salakhutdinov+Geoff-2009}. However, since we impose a
sparsity constraint of 0.1, we have found better results with noise sampled from a Beta
distribution (which is skewed) instead of a Gaussian distribution.  To do this we limit ourselves to
hyper-optimizing the $\beta$ parameter of the distribution, and make the $\alpha$
parameter a function of this $\beta$ such that the mode of the distribution is
equal to our sparsity target of 0.1. Finally, we scale the samples from the
distribution such that the sigmoid of the mode is equal to 0.1. We have
observed that in the case of the STS, introducing noise with a $\beta$ of
approximately 40.1 works bests. We considered values of 1.1, 2.6, 5.1,
10.1, 20.1, 40.1, 80.1 and 160.1. We also considered Gaussian noise and
combinations thereof. We did not try to introduce noise into the SBN or ST
sigmoids since they were observed to have good behavior in this regard.

\subsection{Test-time Thresholds} 

Although injecting noise is useful during training, we found that better
results could be obtained by using a deterministic computation at test time,
thus reducing variance, in a spirit of dropout~\citep{Hinton-et-al-arxiv2012}.
Because of the sparsity constraint with a target of 0.1, simply thresholding
sigmoids at 0.5 does not yield the right proportion of 0's. Instead, we
optimized a threshold to achieve the target proportion of 1's (10\%)
when running in deterministic mode.

\subsection{Hyperparameters}
The noisy rectifier was found to work best with a standard deviation of 1.0
for the gaussian noise. The stochastic half of the Smooth Times Stochastic
is helped by the addition of noise sampled from a beta distribution, as
long as the mean of this distribution replaces it at test time. The Stochastic 
Binary Neuron required a
learning rate 100 times smaller for the gater (0.001) than for the main
part (0.1). The Straight-Through approach worked best without multiplying
the estimated gradient by the derivative of the sigmoid,
i.e. estimating $\frac{\partial L}{\partial a_i}$ by $\frac{\partial L}{\partial  h_i}$ 
where $h_i=\one_{z_i>\sigm(a_i)}$ instead of 
$\frac{\partial L}{\partial h_i}(1-\sigm(a_i))\sigm(a_i) $. Unless indicated otherwise, we
used a learning rate of 0.1 throughout the architecture.

We use momentum for STS, have found that it has a negative effect for SBN, and has little to no effect on the ST and Noisy Rectifier units. In all cases we explored using hard constraints on the maximum norms of incoming weights to the neurons. We found that imposing a maximum norm of 2 works best in most cases.

\fi

\end{document}